\documentclass[letterpaper, 10 pt, conference]{ieeeconf}  

\IEEEoverridecommandlockouts                              

\usepackage{cite}
\usepackage{amsmath,amssymb,amsfonts}
\usepackage{algorithmic}
\usepackage{graphicx}
\usepackage{textcomp}
\usepackage{xcolor}
\usepackage{epstopdf}
\usepackage{ifpdf}
\usepackage{mathrsfs}
\usepackage{color}
\usepackage{graphicx}
\usepackage{subfigure}
\usepackage[ruled,vlined]{algorithm2e}
\usepackage{caption2}

\newtheorem{theorem}{Theorem}
\newtheorem{lemma}{Lemma}

\newtheorem{definition}{Definition}

\newtheorem{remark}{Remark}
\newcommand\correspondingauthor{\thanks{* Corresponding author: {\tt\small qingchen.liu@tum.de}.}}

\begin{document}

\title{\Large \bf Distributed Event- and Self-Triggered Coverage Control with Speed Constrained Unicycle Robots}

\author{Yuni Zhou, Lingxuan Kong, Stefan Sosnowski,  Qingchen Liu\correspondingauthor~and Sandra Hirche
\thanks{
Y. Zhou, S. Sosnowski, Q. Liu and S. Hirche are with the Chair of Information-Oriented Control, Technical University of Munich, Munich, Germany. {\tt\small \{ge69fos, sosnowski, qingchen.liu, hirche\}@tum.de}.\newline
\indent L. Kong is with the Department of Engineering, University of Cambridge, Cambridgeshire, U.K.  {\tt\small leo.kong6577@gmail.com}.\newline
\indent The work of Q. Liu was supported by the European Union’s Horizon 2020 research and innovation programme under the Marie Skłodowska-Curie Grant 754462. The work of S. Hirche was supported by the German Research Foundation (DFG) within the Joint Sino-German research project Control and Optimization for Event-triggered Networked Autonomous Multi-agent Systems (COVEMAS).
}
}

\maketitle

\begin{abstract}
Voronoi coverage control is a particular problem of importance in the area of multi-robot systems, which considers a network of multiple autonomous robots, tasked with optimally covering a large area. This is a common task for fleets of \textit{fixed-wing} Unmanned Aerial Vehicles (UAVs), which are described in this work by a unicycle model with constant forward-speed constraints. We develop event-based control/communication algorithms to relax the resource requirements on wireless communication and control actuators, an important feature for battery-driven or otherwise energy-constrained systems. To overcome the drawback that the event-triggered algorithm requires continuous measurement of system states, we propose a self-triggered algorithm to estimate the next triggering time. Hardware experiments illustrate the theoretical results. 
\end{abstract}


\section{Introduction}



Coverage control is one of the typical problems in the research field of multi-robot systems \cite{cortes2004coverage}. The coverage control problem considers the deployment of a network of mobile autonomous robots to optimally cover a specified area. Applications of coverage control involving multiple networked mobile robots (also referred to as a mobile sensor network) include search and rescue operations, surveillance, environmental monitoring, and exploration of hazardous/inaccessible regions. In many of these applications, an appropriate mobile robot of choice is a fixed-wing UAV. Fixed-wing UAVs can have great endurance in both flight time and range, able to stay in operation for many hours, and can be equipped with a range of different sensors.


The dynamic model of the robots must be considered when designing a distributed control algorithm to solve the coverage control problem. The pioneering work by Cortes et. al \cite{cortes2004coverage} considered single-integrator robot dynamics. In \cite{jiang2015higher}, multiple single-integrator robots covering the same partition of the region are considered. For robots with complex dynamics, e.g. fixed-wing UAVs, unicycle models are a more suitable choice to capture the robot dynamics \cite{seyboth2014collective, sun2015collective}. However, since unicycle models have complex nonholonomic dynamics, control algorithms developed for single-integrator robot models cannot be readily applied. Our recent work \cite{liu2017coverage} has studied the problem of optimal coverage control of unicycle robots and designed the corresponding controller.

Along with the idea in \cite{cortes2004coverage}, follow-up research works \cite{schwager2006distributed, gusrialdi2008voronoi, breitenmoser2010voronoi, gusrialdi2014exploiting, boldrer2019coverage} mainly focus on continuous-time controller design considering practical constraints from sensing devices or robot dynamics/kinematics. It is worth mentioning that the information required for control is the relative position between the robots, because of which wireless communication might be neglected.  However, a multi-robot system benefits from information sharing within a wireless network in which each robot acts as a network node \cite{mavrommati2017real}. For example, 1) the external localization devices can provide more accurate position information; 2) the environment disturbance to the relative information sensing can be avoided. However, a wireless network is usually resource-limited in the sense of communication bandwidth. Thus how to design resource-aware control algorithms with a limited number of information exchanges is crucial for applying distributed control algorithms in practice.  This motivated the development of distributed event-triggered control schemes, see, for example, \cite{dimarogonas2011distributed,wei2018edge,sun2016new,liu2017event}. In an event-triggered algorithm, the control and communication task is executed only when specific trigger conditions are met with respect to local system states. This approach significantly reduces the cost of communication resources since events are generated aperiodically and adaptively.


Event-based coverage control is still a relatively new research topic. In \cite{hayashi2015distributed}, a distributed event-triggered controller was designed for the Voronoi coverage problem. However, their considered robot dynamics are single-integrators and Zeno behavior cannot be excluded. In \cite{nowzari2012self}, a self-triggered mechanism was proposed to achieve both guaranteed and dual guaranteed Voronoi coverage control. The work in \cite{tabatabai2019self, ajina2020asynchronous} follows a similar streamline of \cite{nowzari2012self}. However, their proposed algorithm can be regarded as an extension of discrete-time Lloyd descent, which means, continuous-time robot dynamics/kinematics cannot be involved.

In this work, we present distributed control algorithms for optimal coverage problems with unicycle-type robots, whose forward speed is constrained as a constant. To help achieve the static coverage objective, and for explicit use in the control algorithm to be designed, we define a virtual rotation center for each robot. The event-based scheduler computes an input mismatch according to the robots' states and triggers the control actuation and communication as long as a defined norm of the input mismatch exceeds a state-dependent threshold function. By using Lyapunov analysis, we show that the virtual rotation center for each robot asymptotically converges to a centroidal Voronoi tessellation, thus achieving the optimal coverage objective. Our proposed algorithm can also guarantee Zeno-free triggering by showing a strictly positive lower bound of the inter-event time interval. Meanwhile, we also propose a self-triggered algorithm to overcome the drawback that the event-triggered algorithm requires continuous measurement of system states. The central idea is to estimate the next triggering time, by utilizing the certain bounds of the speed of the evolution states and the information of the threshold function. Hardware experiments have been conducted to illustrate the correctness of the proposed algorithms and demonstrate their performance.

The remainder of this paper is structured as follows: Section II introduces some preliminaries and provides a formal definition of the problem. Section III presents the event-triggered and self-triggered control algorithms with detailed theoretical analysis. The designed algorithms are verified by experiments in Section IV. Finally, we conclude this paper in Section V.

\section{PRELIMINARIES AND PROBLEM STATEMENT}\label{section:2}

\subsection{Notation}  

The set of complex numbers is denoted by $\mathbb{C} $. The imaginary unit is $ i := \sqrt{-1} $. A complex number $ z \in \mathbb{C} $ is denoted as $ \mathbf{Re}(z) + i\mathbf{Im}(z) $, where $\mathbf{Re}(z) $ and $ \mathbf{Im}(z) $ are its real and imaginary part, respectively. The complex conjugate of $z$ is denoted by $ \bar{z} $. For $z_{1}, z_{2} \in \mathbb{C} $, the scalar product is defined by $ \left \langle z_{1}, z_{2} \right \rangle = \mathbf{Re}( \bar{z_{1}}z_{2}) $. The norm of $ z \in \mathbb{C} $ is defined as $ \left \| z \right \| = \left \langle z, z \right \rangle^{\frac{1}{2}} $. 


\subsection{Locational Optimization} 
This section is based on the results of Cortes et. al \cite{cortes2004coverage}. Consider a group of $n$ mobile robots, whose dynamics we define in the sequel, tasked with covering a convex polygon $Q$ in $\mathbb{R}^2$, and let $q$ denote an arbitrary point in $Q$. A distribution density function is a map $\Phi:Q\rightarrow\mathbb{R}_{+}$ that represents a measure of information in $Q$. The location of the $k^{th}$ robot moving in the space $Q$ is denoted by $p_k, k \in \{1, \hdots, n\}$, and $P=(p_1,\ldots,p_n)$ captures the position of all $n$ robots. Given a distribution density function $\Phi(q)$ that is known to all robots, we define a coverage performance function as
\begin{align}\label{eq:performance_function_1}
H(P)=\int_{Q}\min\limits_{k\in\{1,\cdots,n\}}\|q-p_k\|^2\Phi(q)dq.
\end{align}
The quantity $\|q-p_k\|^2$ can be considered as a quantitative measure of how poorly a sensor positioned at $p_k$ can sense the event of interest occurring at point $q$.  

The minimizing operation inside the integral of the above performance function Eq.~\eqref{eq:performance_function_1} induces a Voronoi partition $V(P)=(V_1,\ldots,V_n)$ of the polygon $Q$. Formally, the $k^{th}$ Voronoi partition, $k \in \{1, \hdots, n\}$ is defined as follows: 
\begin{align}\label{eq:voronoi_partition}
V_k=\{q\in Q|~\|q-p_k\|\leq\|q-p_j\|,\forall j\neq k\}
\end{align}
The set of regions ${V_1,\ldots,V_n}$ is called the Voronoi diagram for the generators ${p_1,\ldots,p_n}$. Each Voronoi cell $V_k$ is convex. When two Voronoi regions $V_k$ and $V_j$ are adjacent (i.e., they share an edge), robot $k$, with position $p_k$, is called a neighbor of robot $j$, with position $p_j$, (and vice versa). The set of indexes of the robots which are Voronoi neighbors of $p_k$ is denoted by $N_k$. According to the definition of the Voronoi partition, we have $\min_{k\in\{1,\cdots,n\}}\|q-p_k\|^2 = \|q-p_j\|^2$ for any $q$ inside $V_j$. It follows that, $H(p)$ in \eqref{eq:performance_function_1} can be rewritten as
\begin{align}\label{eq:performance_function_2}
H(P)=\sum_{k=1}^{n}\int_{V_k}\|q-p_k\|^2\Phi(q)dq
\end{align}
The following lemma establishes an important fact regarding the performance function $\eqref{eq:performance_function_2}$.

\begin{lemma}[Lemma 2.1, \cite{cortes2004coverage}] \label{lem:gradient_function}
	The gradient of $H(P)$ is given by
	\begin{align}
	\frac{\partial H}{\partial p_k}=\int_{V_k}\frac{\partial}{\partial p_k}\|q-p_k\|^2\Phi(q)dq 
	\end{align}\hfill $\square$
\end{lemma}
It follows immediately from Lemma \ref{lem:gradient_function} that
\begin{align*}
\frac{\partial H}{\partial p_k}=2p_k\int_{V_k}\Phi(q)dq-2\int_{V_k}q\Phi(q)dq.
\end{align*}
For all $k \in \{1, \hdots, n\}$, define the (generalized) mass $M_{V_k}$, and centroid (or centre of mass) $C_{V_k}$ of the Voronoi partition $V_k$ as:
\begin{align*}
M_{V_k}=\int_{V_k}\Phi(q)dq,~~~C_{V_k}=\frac{1}{M_{V_k}}\int_{V_k}q\Phi(q)dq.
\end{align*}
With these definitions, the gradient can now be expressed as
\begin{align}\label{eq:gradient_fucntion}
\frac{\partial H}{\partial p_k}=2M_{V_k}(p_k-C_{V_k})
\end{align}
Critical points of $H$ are those in which every robot is at the centroid of its Voronoi cell. A Voronoi configuration corresponding to critical points of $H$ is called a Centroidal Voronoi Tessellation. The objective of this paper is to design a distributed controller for each robot $k$. Then, collectively, the robots can move into positions that establish a Centroidal Voronoi Tessellation.

\subsection{Unicycle Model with Constant Speed}

We choose to model each robot as a unicycle. Unicycle dynamics are appropriate for capturing fixed-wing UAVs or ground-based vehicles. The unicycle model for the $k^{th}$ robot, $k \in \{1, \hdots, n\}$, is given as
\begin{align*}
\dot{x}_k(t)&=v_k\cos(\theta_k(t))\notag\\
\dot{y}_k(t)&=v_k\sin(\theta_k(t))\notag\\
\dot{\theta}_k(t)&=u_k(t)
\end{align*}
where $x_k(t)\in\mathbb{R}$, $y_k(t)\in\mathbb{R}$ are the coordinates of robot $k$ in the real plane and $\theta_k(t)$ is the heading angle (i.e. the forward facing direction of the robot) at time $t$. The forward velocity $v_k\in\mathbb{R}_+$ is by definition strictly positive. In this paper, we assume that it is \textit{fixed/constant/nonidentical}, and cannot be designed as a control input. In practice, the cruising speed may be the most fuel-efficient speed for the robot (which is desirable for long periods of surveillance) or optimal against some other performance measure. The control input $u_k(t)$ is to be designed for steering the orientation of robot $k$. When there is no risk of confusion, we drop the argument $t$.


For analysis purposes, we express the position of robot $k$, $(x_k, y_k)$, in the complex plane using the complex notation $r_k=x_k+iy_k := \|r_k\|e^{i\theta_k}\in\mathbb{C}$. Then the above unicycle dynamics can be reformulated as
\begin{align}\label{eq:complex_unicycle_model}
\dot{r}_k(t)&=v_ke^{i\theta_k(t)}\notag\\
\dot{\theta}_k(t)&=u_k(t)
\end{align}
In the complex plane, robot $k$'s instantaneous center of the circular orbit $c_k$ can be given by
\begin{equation}
c_k(t)  = r_k(t) + \frac{v_k}{ \dot{\theta}_k(t)} i e^{i\theta_k(t)}
\end{equation}

To help achieve the control objective, and for the explicit use in the control algorithm to be designed, we define for each robot $k$ a virtual center $z_k$, as
\begin{equation}\label{eq:virtual_centre}
z_k(t) = r_k(t)+\frac{v_k}{\omega_0}ie^{i\theta_k(t)}.
\end{equation}
Here, $\omega_0$ is a non-zero constant. Differentiating both sides of \eqref{eq:virtual_centre}, and substituting in \eqref{eq:complex_unicycle_model}, the dynamics of the virtual center can be written as:
\begin{align}\label{eq:virtual_centre_dynamic}
\dot{z}_k(t) &= v_ke^{i\theta_k(t)}-\frac{v_k}{\omega_0}e^{i\theta_k(t)}u_k(t),~~k=1,\cdots,n
\end{align}

\subsection{Problem Statement} 
Given a convex polygon $Q \in \mathbb{C}$, suppose there are $n$ unicycle robots with dynamics given by \eqref{eq:complex_unicycle_model}, for $k = 1, \hdots, n$. Let $v_k$ be the constant forward speed for robot $k$, with $v_k$ not necessarily equal to $v_j$ for $ k \neq j,~k,~j \in \{1, \hdots n\}$, and $z_k$ be the virtual center as defined in \eqref{eq:virtual_centre}. Let $Z$ be the stacked column vector of all virtual centers $z_k$, i.e. $Z = (z_1, \cdots, z_n)$. The objective is to design the control input $u_k$ to steer the virtual center $z_k$ of each unicycle robot to a desired Voronoi centroid $C_{V_k}$ so that the polygon $Q$ can be covered optimally in the sense of minimizing the performance function
\begin{align}\label{eq:performance_function_3}
H_{V}(Z)=\sum_{k=1}^{n}\int_{V_k}\|q-z_k\|^2\Phi(q)dq.
\end{align}
That is, we desire
\begin{equation}\label{eq:zk_objective}
\lim_{t\to\infty} z_k(t) = C_{V_k}(t)\,~~\forall\;k = 1, \hdots, n
\end{equation}

Once the virtual center has arrived at the centroid, i.e. \eqref{eq:zk_objective} has been achieved, each robot will orbit about the centroid with the fixed angular speed $\omega_0$. That is, $\lim_{t\to\infty} \dot{\theta}_k = \omega_0$ for all $k$, which when combined with the existing assumption that $v_k$ is constant, implies robot $k$ is executing a steady-state circular orbit. Since $v_k$ and $v_j$ are not necessarily equal for $k \neq j$ and $k,j \in \{1, \hdots, n\}$, it follows that the robots' steady-state orbit radii are not necessarily equal.

Liu et. al \cite{liu2017coverage} proposed a continuous controller:
\begin{equation} \label{eq:continuous control law}
u_{k}(t) = w_{0} + \gamma w_{0}\left \langle z_{k}(t)-C_{V_{k}}(t), v_{k}e^{i\theta _{k}(t)} \right \rangle
\end{equation}
where $\gamma > 0$ is a positive control gain. Building on this previous approach, we aim to design event-triggered and self-triggered control algorithms to achieve optimal coverage, which should satisfy the following requirements:
\begin{enumerate}
	\item The algorithm is distributed in the sense that only neighbor information is received for each robot to update the controller. 
	\item Zeno behavior should be completely excluded for each robot in implementing the control algorithm.
	\item The self-triggered algorithm does not require any continuous measurement of the states.
\end{enumerate}

The following lemma will later be used for stability analysis.
\begin{lemma} [Lemma 3, \cite{liu2020network}] \label{lem:stability analysis}
Consider a continuously differentiable function $ g : \mathbb{R}^{\geq 0} \rightarrow \mathbb{R}^{\geq 0} $. If there exist continuous
functions $ \gamma : \mathbb{R}^{\geq 0} \rightarrow \mathbb{R}^{+}  $  and $ \beta : \mathbb{R}^{\geq 0} \rightarrow \mathbb{R}^{+} $ satisfying $ \dot{g}(t) \leq -\gamma (t)g(t) + \beta (t) $, then 
\begin{equation}
	g(t) \leq e^{-\int_{0}^{t}\gamma (s)ds}g(0) + \int_{0}^{t}e^{-\int_{s}^{t}\gamma (r)dr}\beta (s)ds
\end{equation}
Furthermore, the following statements hold:
\begin{itemize}
	\item If $ \int_{0}^{\infty }\gamma (t)dt = \infty $ and $ \underset{t\rightarrow \infty }{lim}\frac{\beta (t)}{\gamma (t)}= 0 $, then
	$\underset{t\rightarrow \infty }{lim}g(t)= 0$.

	\item If $ \int_{0}^{\infty }\gamma (t)dt = \infty $ and $ \underset{t\rightarrow \infty }{lim}sup \frac{\beta (t)}{\gamma (t)}< \infty $, then $\left \{ g(t) \right \}_{t\geq 0} $ is bounded.
\end{itemize} 
\end{lemma} 

%
%
\begin{remark}
Since fixed wing UAVs can be set to work in different heights, collision avoidance is not a strict requirement here and subject to future work.
\end{remark}

\section{Problem analysis and algorithm framework}\label{section:3}

\subsection{Event-Triggered Controller Design} \label{section event triggered problem statement}

Let the event time instants for robot $k$ be denoted as $t_{0}^{k}=0,t_{1}^{k},...,t_{l}^{k},...$, where $l\in \mathbb{Z}$ denoting the set of all nonnegative integers. In this paper, the execution rule is based on actuation errors rather than measurement errors. The control input mismatch for robot $k$ is defined as
\begin{equation} \label{eq: actuation mismatch}
e_{k}(t) = u_{k}(t_{l}^{k})  - u_{k}(t), ~t\in [t_{l}^{k}, t_{l+1}^{k}) 
\end{equation}
Every time an event is triggered, $e_{k}(t)$ is reset to be equal to zero. The event-based control input is described by 
{\small
\begin{equation} \label{eq:event_based control update}
u_{k}(t_{l}^{k}) =
 w_{0} + \gamma w_{0}\left \langle z_{k}(t_{l}^{k})-C_{V_{k}}(z_{k}(t_{l}^{k}), z_{j}(t_{l^{j}}^{j})), v_{k}e^{i\theta _{k}(t_{l}^{k})} \right \rangle
\end{equation}
}where $j \in \mathcal{N}_{k}, l^{j} = argmin_{a\in \mathbb{N}}: t\geq t_{a}^{j}$. For $t\in [t_{l}^{k}, t_{l+1}^{k})$, $t_{l^{j}}^{j}$ is the last event time of robot $j$. Each robot takes into account the last update value of each of its neighbors in its control law. According to the definition of the actuation error \eqref{eq: actuation mismatch} we obtain
\begin{equation} \label{eq:0 t and error ersatz diacrete control update}
u_{k}(t_{l}^{k}) = u_{k}(t) + e_{k}(t)
\end{equation}

We define an auxiliary variable 
\begin{equation} \label{eq:Hilfe Variable fuer control input}
g_{k}(z_{k}(t)) = \left \langle z_{k}(t)-C_{V_{k}}(t), v_{k}e^{i\theta _{k}(t)} \right \rangle
\end{equation}

By substituting \eqref{eq:0 t and error ersatz diacrete control update} into the dynamic equation of each robot's virtual center \eqref{eq:virtual_centre_dynamic}, we obtain 
\begin{equation}\label{eq:event_based virtual_centre_dynamic}
\dot{z}_k(t)
 =   - v_k e^{i\theta_k(t)}\left ( \gamma g_k(z_{k}(t)) + \frac{e_{k}(t)}{\omega_0} \right )
\end{equation}

We consider the following trigger function for each robot:
\begin{equation} \label{eq: trigger function}
f_{k}= \underbrace{\left | e_{k}(t) \right |}_{\text{error}} - \underbrace{\sigma \gamma \omega_0 \left | g_{k}(z_{k}(t)) \right | - \mu_{k}(t)}_{\text{comparison threshold}},
\end{equation}
where $\gamma > 0$, $\omega_0 > 0$, $ 0 <\sigma <1 $, $\mu_{k}(t) = \gamma \omega_0 e^{-\alpha _{k}t}$ with $ 0< \alpha_{k} < 1 $. The error term is reset to zero whenever $f_{k} = 0$, i.e., the event is triggered. We now present the main result of the event-triggered control algorithm.

\begin{theorem} \label{Theorem: event based triggering}
	Consider a group of n unicycle-type robots with constant, non-identical speeds modeled by \eqref{eq:complex_unicycle_model} and driven by controller \eqref{eq:event_based control update}. The controller and the trigger function only require neighboring information to update. If each robot updates its input when the designed state-dependent trigger function $f_{k} = 0$, then their virtual centers converge asymptotically to the set of centroidal Voronoi tessellation (local minimum equilibrium for the coverage performance function) on $Q$ and no robot exhibits Zeno behavior. 
\end{theorem}

\begin{proof}
See Appendix.
\end{proof}

\subsection{Self-Triggered Controller Design}
The intuitive idea for a self-triggered algorithm is to estimate the next updating time instant, by using the bounds of the actuation error $|e_{k}(t)|$ and the bounds of the comparison threshold. We first illustrate how to compute the bound of $|e_{k}(t)|$. Since the actuation error $ \left | e_{k}(t) \right | = \left | u_{k}(t_{l}^{k})  - u_{k}(t) \right |$, 
the triggering rule can be rewritten as
\begin{equation}
\left | u_{k}(t_{l}^{k})  - u_{k}(t) \right | \leq \sigma \gamma \omega_0 \left | g_{k}(z_{k}(t)) \right | + \mu_{k}(t)
\end{equation}

Note that between the event-triggered control updates, the control input is held constant via the Zero-order hold technique. This observation motivates us to provide an estimation for $u_{k}(t)$, if $ t \in [  t_{l}^{k}, min\left \{  t_{l+1}^{k}, \underset{j \in N_{k}}{min}~t_{l^{''}}^{j}\right \})$, using:
\begin{equation}
u_k(t) = \dot{u}_{k}(t_{l}^{k}) (t - t_{l}^{k}) + u_{k}(t_{l}^{k}) 
\end{equation}
and the time derivative of $u_k(t)$ is expressed as follows: 
\begin{equation} \label{time derivative of the controller}
\begin{split}
\dot{u}_{k}(t_{l}^{k})  &=  \gamma w_{0} \left \langle \dot{z_{k}}(t_{l}^{k})-\dot{C}_{V_{k}}(t_{l}^{k}), v_{k}e^{i\theta _{k}(t_{l}^{k})} \right \rangle \notag \\
& ~~~ + \gamma w_{0}\left \langle z_{k}(t_{l}^{k})-C_{V_{k}}(t_{l}^{k}), iv_{k}u_{k}e^{i\theta _{k}(t_{l}^{k})} \right \rangle
\end{split}
\end{equation}
where $ l^{''} \triangleq arg \underset{m \in \mathbb{N}: t_{l}^{k}\leq t_{m}^{j}}{min}\left \{ t_{m}^{j} - t_{l}^{k}  \right \} $; $t_{m}^{j}$ represents the update time instants of all the robot $k$'s neighbors after time instant $t_{l}^{k}$. $ t_{l^{''}}^{j} $ represents the latest update time instant for any robot $k$'s neighbors after time instant $t_{l}^{k}$; $\underset{j \in N_{k}}{min}~t_{l^{''}}^{j}$ represent the latest update time instant among the robot $k$'s neighbors, which is the nearest to the time instant $t_{l}^{k}$. 
Hence $min\left \{  t_{l+1}^{k}, \underset{j \in N_{k}}{min}~t_{l^{''}}^{j}\right \}$ is the next time when the control input is updated. Thus, the triggering rule is equivalent to 
\begin{equation}
\begin{split}
\left | u_{k}(t_{l}^{k})  - u_{k}(t) \right | &= \left | \dot{u}_{k}(t_{l}^{k}) (t - t_{l}^{k}) \right | \notag \\
& \leq  \sigma \gamma \omega_0 \left | g_{k}(z_{k}(t)) \right | + \mu_{k}(t)
\end{split}
\end{equation}
With the auxiliary variable defined in \eqref{eq:Hilfe Variable fuer control input}, we have 
{\small
\begin{equation*}
\left | \dot{u}_{k}(t_{l}^{k}) (t - t_{l}^{k}) \right | \leq \sigma \left |  \dot{u}_{k}(t_{l}^{k}) (t - t_{l}^{k}) + u_{k}(t_{l}^{k}) - \omega_0 \right | + \mu _{k}(t_{l}^{k})
\end{equation*}
}
Let $\xi_{l}^{k} = \xi (t) = t - t_{l}^{k} $, $ t \geq t_{l}^{k} $. Since $\xi_{l}^{k} \geq 0 $, we have the triggering condition
\begin{equation}
\xi_{l}^{k}(1- \sigma )\left | \dot{u}_{k}(t_{l}^{k})  \right | \leq \sigma \left |  u_{k}(t_{l}^{k}) - \omega_0 \right |  + \mu _{k}(t_{l}^{k})
\end{equation}
For $ \left | \dot{u}_{k}(t_{l}^{k})  \right | \neq 0 $, we have the positive inter-execution time interval
\begin{equation} \label{inter-event time interval}
\xi_{l}^{k} = \frac{\sigma \left |  u_{k}(t_{l}^{k}) - \omega_0 \right |  + \mu _{k}(t_{l}^{k})}{(1- \sigma )\left | \dot{u}_{k}(t_{l}^{k})  \right |}
\end{equation}
where $ 0 <\sigma <1 $. In the determination of the inter-execution time, no continuous measurement of the neighbor's states, and no continuous computation of the control input is needed.
The explanation of the self-triggered rule for each robot can be summarized as: if no state of any neighbors is received in the time interval $(t_{l}^{k}, t_{l}^{k} + \xi_{l}^{k} ) $ and there is a positive inter-execution time interval $\xi_{l}^{k} \geq 0$ such that the triggering condition $\xi_{l}^{k}\left | \dot{u}_{k}(t_{l}^{k}) \right | =  \left | \sigma \dot{u}_{k}(t_{l}^{k}) \xi_{l}^{k}+ \sigma u_{k}(t_{l}^{k}) - \sigma \omega_0 \right | + \mu _{k}(t)$ is satisfied, then the next update time $t_{l+1}^{k}$ takes place at most $ \xi_{l}^{k}$ time units after $t_{l}^{k}$. In other words, $t_{l+1}^{k} \leq t_{l}^{k} + \xi_{l}^{k} $. Otherwise, if the robot $k$'s control input is updated due to an update of one of its neighbors for $(t_{l}^{k}, t_{l}^{k} + \xi_{l}^{k} ) $, then the triggering condition needs to be re-checked and the next triggering time is new determined. 

We also provide a pseudo-algorithm for illustrative purpose, see Algorithm \ref{alg:a1}. Here, we let $\tau_{p}$ represent the temporary selection for the next triggering time instant $ t_{l+1}^{k} $. At time instant $t_{l+1}^{k}$, robot $k$ will measure all the neighbors' states, send its own states to all the neighbors and its control input will be updated. 
\begin{algorithm}\label{alg:a1}
\SetAlgoLined
	\KwData{the latest triggering time $ t_{l}^{k} $, the latest triggering time among the neighbors after $ t_{l}^{k}$ : $\underset{j \in N_{k}}{min}~t_{l^{''}}^{j}$ }
	\KwResult{determination of the next triggering time $ t_{l+1}^{k} $ }
    \textbf{Initialization}: $ t_0 \leftarrow t_{l}^{k}$; $u_{k}(t_0)$; $ \xi_0 \leftarrow \xi_{l}^{k}$; $\dot{u}_{k}(t_0)$  \;
	\While{$\dot{u}_{k}(t_0)  \neq 0$}{
		$\tau_{p}\leftarrow t_p + \xi_{p}$\;
		\eIf{ any neighbor of robot $k$ updates its control law in $ t \in ( t_p, \tau_{p} )$ }{
			 $p \leftarrow p + 1$, $t_p \leftarrow \underset{j \in N_{k}}{min}~t_{l^{''}}^{j}  $ \;
			Update $ u_{k}(t_p)$, $\dot{u}_{k}(t_p)$, $\xi_{p}$ \;
			Continue\;
		}{
			 $ t_{l+1}^{k} \leftarrow \tau_{p}  $\;
			 Break\;
		}
	}

	\Return{$ t_{l+1}^{k}$} 
	\caption{Determination of the next triggering time $ t_{l+1}^{k} $ }
\end{algorithm}
Applying the designed self-triggering rule, we have the following theorem:

\begin{theorem} \label{Theorem: self triggered}
	Consider a group of n unicycle-type robots with constant, non-identical speeds modeled by \eqref{eq:complex_unicycle_model} and driven by the controller \eqref{eq:event_based control update}. If each robot decides when to update its control input according to Algorithm \ref{alg:a1}, then their virtual centers converge asymptotically to the set of centroidal Voronoi tessellation on $Q$ and no robot exhibits Zeno behavior. 
\end{theorem}

\begin{proof}
\textit{ 1) Stability analysis:} This part is similar to the proof of the event-based triggering theorem. The virtual centers of all unicycles will asymptotically converge to the set of centroidal Voronoi tessellations on $Q$. 

\textit{ 2) Absence of Zeno behavior:} When $\dot{u}_{k}(t_{l}^{k}) \neq 0$, there always exists a positive inter-execution time $\xi _{k} $, then we can conclude that there is no Zeno behavior in the system. When $ \dot{u}_{k}(t_{l}^{k}) = 0 $, it indicates that the control input has reached the desired value $ \omega_0$ and the centroidal Voronoi tessellation is achieved.
\end{proof}

\section{Experiments}\label{section:5}
In this section, the results of hardware experiments are provided to illustrate the performance of the proposed event-triggered and self-triggered algorithms. A video is available in our supplementary materials.

\subsection{Experimental setup}
\begin{figure}[t]
\centering
\includegraphics[width=8cm]{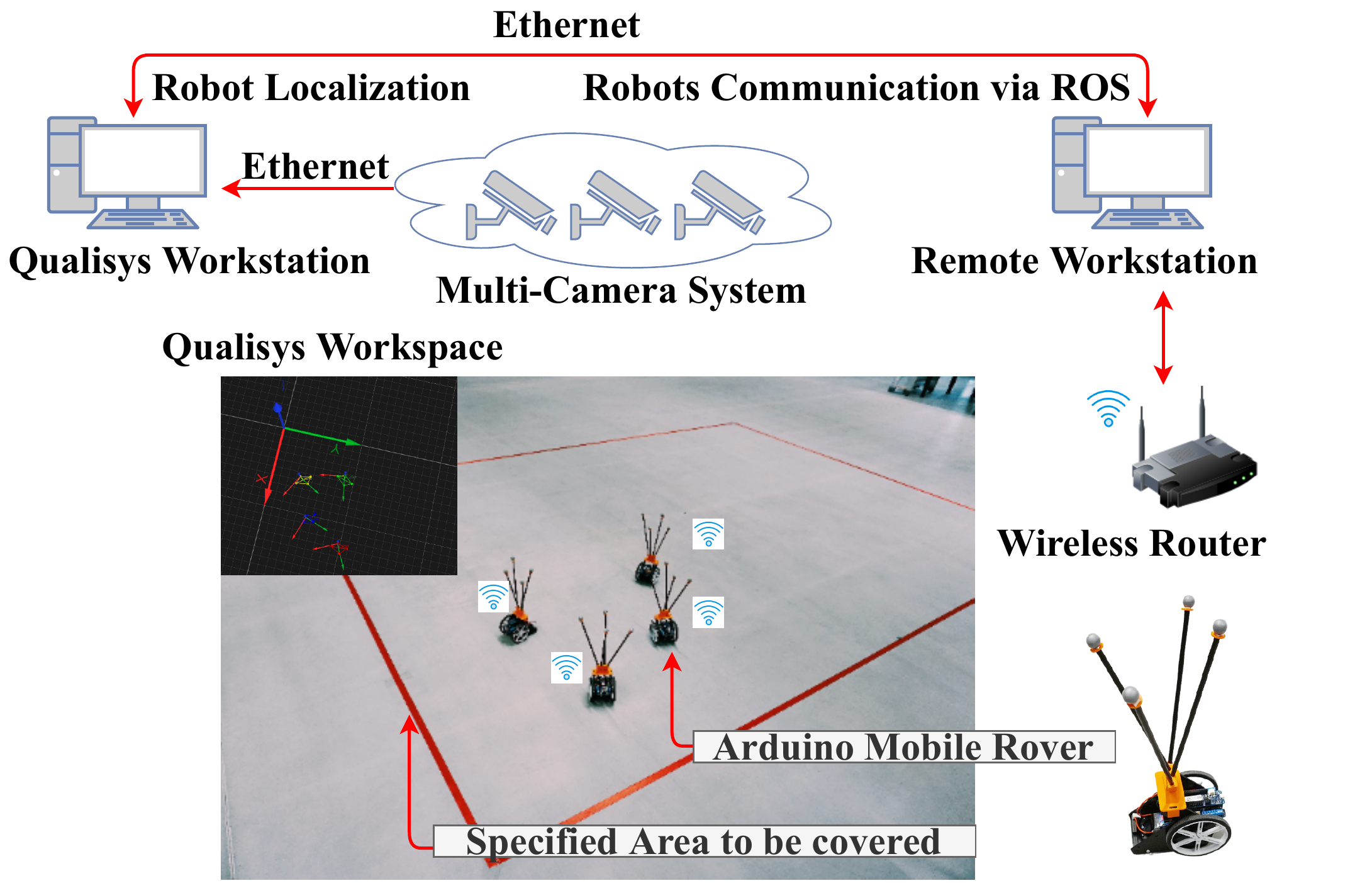}
\caption{Experimental setup}\label{fig.structure}
\end{figure}

The structure of our experimental setup is illustrated in Fig. \ref{fig.structure}. The system comprises a Qualisys multi-camera system, a remote workstation and four Arduino mobile rovers. The Qualisys multi-camera system provides high precision position and orientation information for the ground robots with up to 300Hz refresh rate. The remote workstation is a Lenovo ThinkPad laptop (Intel i5-6200U CPU and 8GB RAM), running Ubuntu 16.04 and robot operating system (ROS) kinetic. The mobile robots used in our experiment are Arduino mobile rover with Wi-Fi access, which is a differentially driven robot\footnote{The complete description of the Arduino mobile rover is referred to Arduino Engineering Kit.}. The remote workstation and the Arduino mobile rovers are communicating in a Wi-Fi network, with an ASUS RT-AC1750 router. 

Our control framework is two-layer. The lower layer runs on the Arduino mobile rover, in which the program interprets the input commands (forward speed and angular speed) to the respective rotational speeds of the two wheels and uses PID controllers to track them. The upper layer runs on the remote workstation, which consists of four ROS nodes. Each ROS node communicates with its neighbors to locally compute the Voronoi partitions (according to \cite{hadjicostis2003distributed}) and the event-triggered and self-triggered control input \eqref{eq:event_based control update}. Then the control commands are transmitted to the Arduino mobile rovers via UDP protocol. 

We provide two separate hardware experiments to demonstrate the performances of both event-triggered and self-triggered algorithms applied in a four-unicycle group. In both experiments, the coverage area $Q$ is rectangular in the size of $4.0\text{m}\times 2.8\text{m}$. The density function $\Phi(q)$ is assumed to be $1$ in the experiments. The angular frequency $\omega_0$ defined in \eqref{eq:virtual_centre} is selected as $0.536\text{rad/s}$ for all robots. The robots' forward speeds are set as $v_k=0.16\text{m/s}$ for all $k$. The selected values of the forward speed and angular speed meet the hardware constraints of the Arduino mobile rover. The initial positions and initial orientations of the robots are manually chosen such that the initial virtual centers locate inside the coverage region. 

\subsection{Experimental results}

\begin{figure}[t]
\centering
\includegraphics[width=8cm]{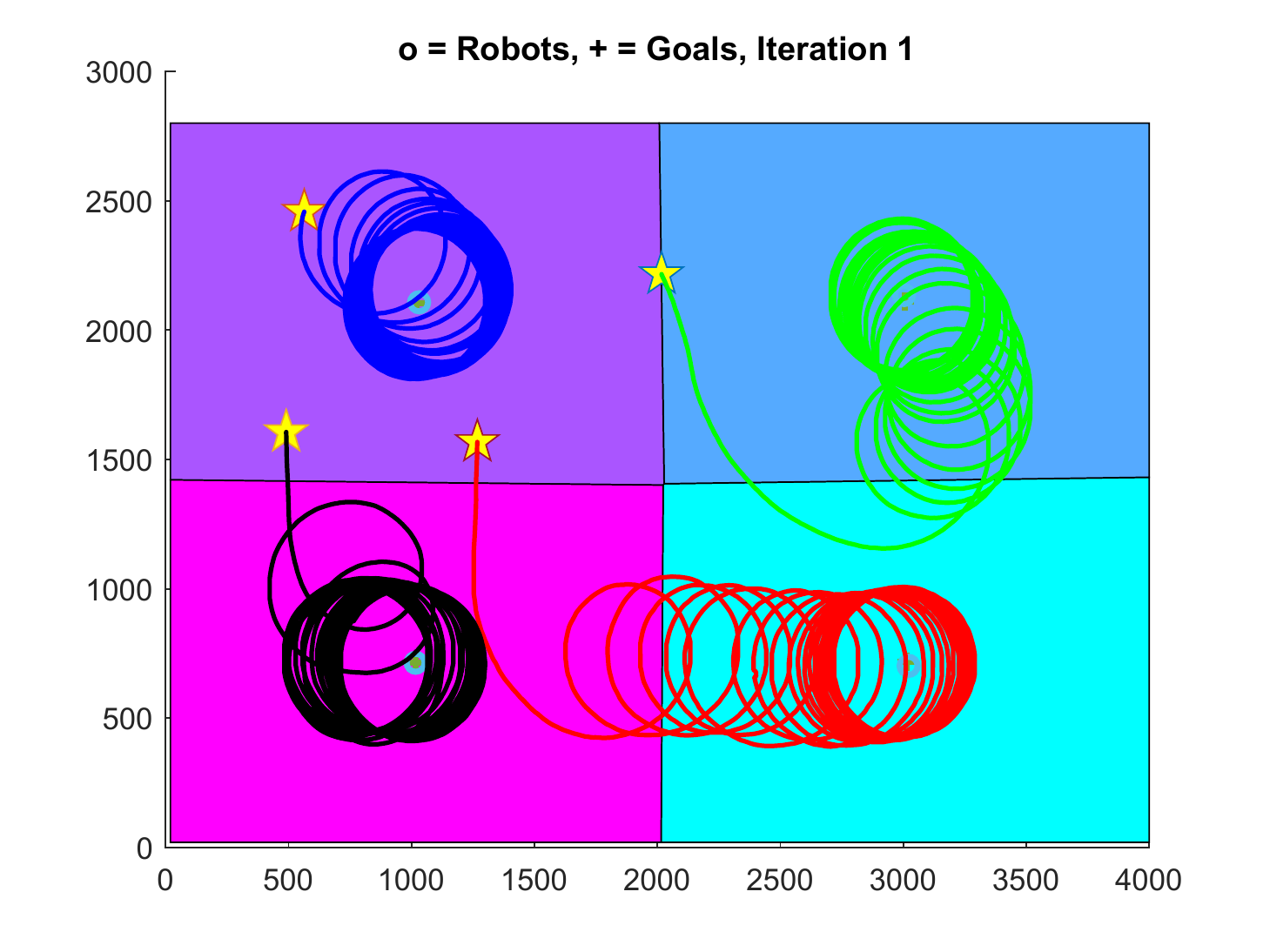}
\caption{ Coverage performance using event-triggered controller. The coverage area is in the size of $4.0m\times 2.8m$. The yellow stars denote the initial positions of the robots. The solid lines with different colors represent the trajectories of each robot. All robots move in a circular orbit concerning their virtual center. }
\label{fig:f_event}
\end{figure}

\begin{figure}[t]
\centering
\includegraphics[width=7.5cm]{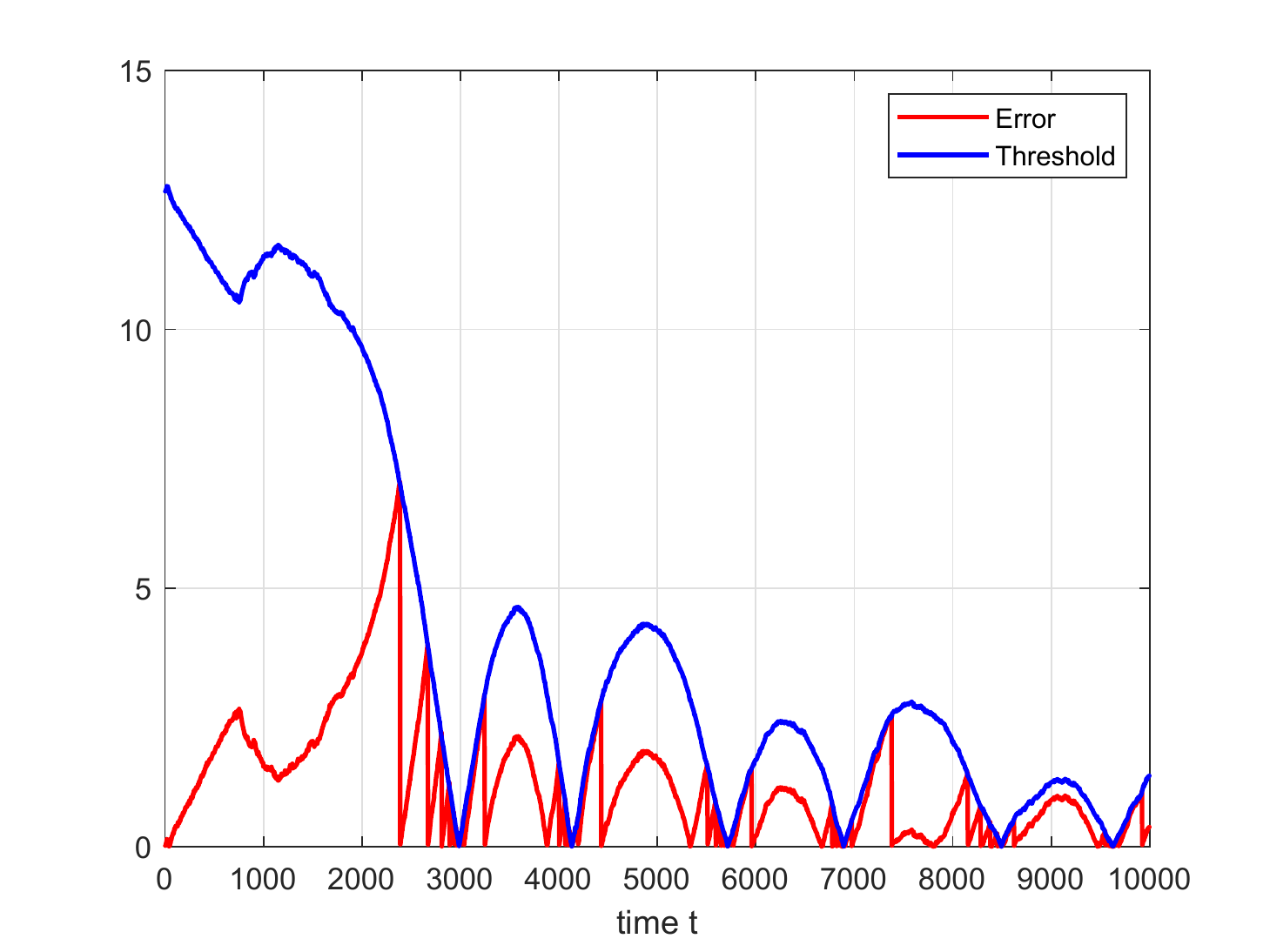}
\caption{Triggering behavior during the control process (time: ms). Error (red) vs. comparison threshold (blue). }\label{fig:f_error_threshold}
\end{figure}

\begin{figure}[t]
\centering
\includegraphics[width=7.5cm]{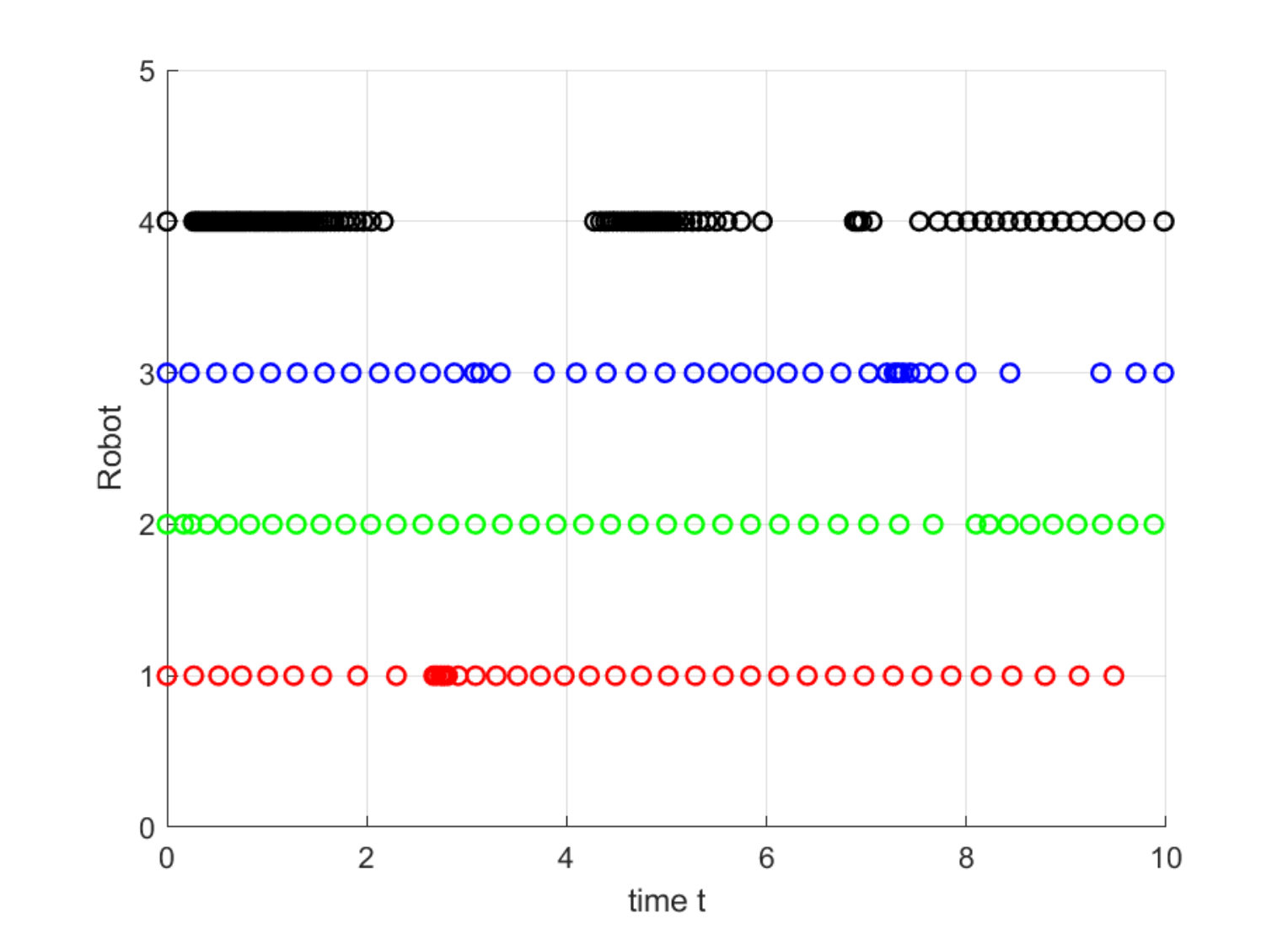}
\caption{Triggering time instants for the self-triggered algorithm. The red, green, blue and black dots denote the triggering times for robot 1, robot 2, robot 3 and robot 4, respectively. For illustration purpose, we select the time period between 0s to 10s with the triggering time instants for the four robots.}\label{fig:f_self_trigger}
\end{figure}

The experimental results of the event-triggered controller is shown in Fig. \ref{fig:f_event}, from which one can see that the optimal coverage objective has been successfully achieved. All virtual centers move asymptotically to the set of Centroidal Voronoi Tessellation. All robots move in a circular orbit concerning their virtual center. The trajectories of different robots are represented by solid lines with different colors. We note that the depicted trajectories are the raw data from Qualisys. The evolution of the event-triggered error and the comparison threshold in \eqref{eq: trigger function} is shown in Fig. \ref{fig:f_error_threshold}. When the value of the error reaches the threshold, the error is reset to zero and an event is triggered. The control input is updated according to Theorem \ref{Theorem: event based triggering}. The red line represents the evolution of the error. It stays below the designed threshold which is shown by the blue line in Fig. \ref{fig:f_error_threshold}.  

The second experiment illustrates the performance of the self-triggered algorithm. Due to the page limit, we omit the figure about the trajectories of the robots, since it is quite similar to the event-triggered case. The self-triggered control law can guarantee that the inter-event time intervals are lower bounded by a strictly positive constant. Thus, the Zeno behavior is excluded, which can be observed from Fig. \ref{fig:f_self_trigger}.  

From both experiments, it is observed that the optimal coverage objective can be achieved. The asynchronous feature of both the event-triggered and self-triggered algorithms are verified.

\section{CONCLUSION}\label{section:6}

In this work, we develop event-/self-triggered algorithms for coverage control problems with a group of unicycle-type robots to facilitate the efficient usage of the network resource. We firstly propose an event-based control algorithm, which generates events and control updates when the actuation error diverges to a designed threshold. Zeno behavior can be completely excluded by showing that strictly positive inter-event time intervals exist. Then we propose a self-triggered control algorithm to relax the strong assumption that continuous measurement is required for each robot, in which the core idea is to provide an estimation for the next update time instant based on the state and control value at the latest update time. Hardware experiments have been conducted to verify the correctness and performance of our proposed algorithms.

\section*{APPENDIX}

\begin{proof}
The proof is divided into two parts. In the first part, the stability of the system \eqref{eq:event_based virtual_centre_dynamic} will be analyzed. In the second part, we will show Zeno behavior is excluded.
\textit{ 1) Stability analysis:} By using Lem.\ref{lem:gradient_function} and Eq.\eqref{eq:gradient_fucntion}, the time derivative of the performance function \eqref{eq:performance_function_3} along the trajectory of system \eqref{eq:event_based virtual_centre_dynamic}, is
\begin{equation} \label{eq:prove the time derivate of the Hv}
\begin{split}
\dot{H}_{V}(Z) &= \sum_{k=1}^{n}2M_{V_{k}}\left \langle z_{k}(t)-C_{V_{k}}(t), \dot{z}_{k}(t) \right \rangle \notag \\
& =- \sum_{k=1}^{n}2M_{V_{k}}\left ( \gamma g_{k}^{2}({z}_{k}(t)) + \frac{e_{k}(t)}{\omega_0}g_{k}({z}_{k}(t))  \right )  \notag 
\end{split}
\end{equation}
Because $\gamma$, $\omega_0$ and $M_{V_{k}}$ are strictly positive, if $\left | e_{k}(t) \right | \leqslant \sigma \gamma \omega_0 \left | g_{k}(z_{k}(t)) \right |$,$ 0 <\sigma <1 $, then $\dot{H}_{V}(Z)  \leq 0 $ for all $t \geq 0 $. Because $\mu_{k}(t) > 0$, it is evident that $ \left | e_{k}(t) \right | \leqslant \sigma \gamma \omega_0 \left | g_{k}(z_{k}(t)) \right | + \mu_{k}(t) $. Each robot updates its input when $f_{k} = 0$, then we have 
\begin{equation}  \label{eq: upper bounded of Hv}
\begin{split}
 \dot{H}_{V}(Z) & \leq - \sum_{k=1}^{n}2\gamma M_{V_{k}}g_{k}^{2}(z_{k}(t)) \notag \\
&~~~ +\sum_{k=1}^{n}\frac{2}{\omega_0} M_{V_{k}}\left | e_{k}(t) \right |\left | g_{k}(z_{k}(t)) \right | \notag \\
& \leq \sum_{k=1}^{n}(\sigma -1)2\gamma M_{V_{k}}g_{k}^{2}(z_{k}(t)) \notag \\ 
&~~~ + \sum_{k=1}^{n}2\gamma M_{V_{k}} e^{-\alpha _{k}t}O\left | \cos \psi \right |
\end{split}
\end{equation}
where $\psi $ is the angle between the vectors $ z_k - C_{V_{k}} $ and $ v_ke^{i\theta_k} $, and $O = \left \| z_k - C_{V_{k}} \right \|\left \| v_ke^{i\theta_k} \right \| $. 

Note that we can rewrite the performance function \eqref{eq:performance_function_3} as follows (see \cite{cortes2004coverage} for more details):
\begin{equation}
H_{V}(Z) = \sum_{k=1}^{n}J_{V_{k}, C_{V_{k}}} + \sum_{k=1}^{n}M_{V_{k}}\|z_k - C_{V_{k}} \|^2
\end{equation}
where $J_{V_{k}, C_{V_{k}}} \in\mathbb{R}^{+} $ is the polar moment of inertia of the Voronoi cell $V_{k}$ about its centroid $C_{V_{k}}$. Define 
\begin{equation}
H_{V, 1}(Z) = \sum_{k=1}^{n}J_{V_{k}, C_{V_{k}}}, H_{V, 2}(Z)= \sum_{k=1}^{n}M_{V_{k}}\|z_k - C_{V_{k}} \|^2
\end{equation}
Then we have
\begin{equation} 
\begin{split}
&\dot{H}_{V}(Z) \\ 
& = \dot{H}_{V, 2}(Z)   \\
& \leq  -\sum_{k=1}^{n}(1 - \sigma)2\gamma M_{V_{k}}\left \| z_k - C_{V_{k}} \right \|^{2}\left \| v_ke^{i\theta_k} \right \|^{2}\cos ^{2}\psi   \\
& ~~~+ \sum_{k=1}^{n}2\gamma e^{-\alpha _{k}t} M_{V_{k}} \left \| z_k - C_{V_{k}} \right \|\left \| v_ke^{i\theta_k} \right \|\left | \cos \psi \right |
\label{inequality: separat Hv1, Hv2, g }
\end{split}
\end{equation}
Define $ \varrho $ as the set of $Z, \theta $ for which $ \dot{H}_{V}(Z) = 0 $. We consider the following two cases: 1) $  \cos \psi = 0, z_k \neq C_{V_{k}}  $; 2)  $  \cos \psi \neq 0 $. \\

\noindent\textbf{Case 1:}  $  \cos \psi = 0, z_k \neq C_{V_{k}} $. Then we have 
\begin{equation}
\dot{H}_{V}(Z) \leq 0
\end{equation}
 $  \cos \psi = 0 $ means $ z_k - C_{V_{k}} \perp  v_ke^{i\theta_k} $. Then from Eq.\eqref{eq:continuous control law} it is concluded that $u_k = \omega_0 $, which implies $ \dot z_k = 0 ~\forall k \in \left \{ 1,..., n \right \}$. The vector $  v_ke^{i\theta_k} $ changes its direction with angular frequency $ \omega_0 $. Therefore $ v_ke^{i\theta_k} $ is non-constant but $ z_k - C_{V_{k}} $ is constant. 
 It is to verify that the set $\left \{ (Z, \theta ) : z_{k} - C_{V_{k}}\perp v_ke^{i\theta_k},z_{k} - C_{V_{k}}\neq 0, \forall k \in \left \{ 1,...,n \right \} \right \}$ is not the equilibrium of the system. \\
 
\noindent\textbf{Case 2:} $  \cos \psi \neq 0 $.\\
Let
\begin{eqnarray}
h_k(t) & =& M_{V_{k}}\|z_k - C_{V_{k}} \|^2  \notag\\
\kappa _k (t) &=& (1 - \sigma)2\gamma \left \| v_ke^{i\theta_k} \right \|^{2}\cos ^{2}\psi \notag\\
\beta _k(t) &=& 2\gamma e^{-\alpha _{k}t} M_{V_{k}} \left \| z_k - C_{V_{k}} \right \|\left \| v_ke^{i\theta_k} \right \|\left | \cos \psi \right | \notag
\end{eqnarray}

Then we can rewrite the Eq.\eqref{inequality: separat Hv1, Hv2, g } as follows:
\begin{equation}
\sum_{k=1}^{n} \dot h_k(x) \leq  \sum_{k=1}^{n} (-\kappa _k(t) h_k(x) + \beta _k(t)  )
\end{equation}

\begin{equation}
\int_{0}^{\infty }\kappa _k (t)= \int_{0}^{\infty }(1 - \sigma)2\gamma \left \| v_ke^{i\theta_k} \right \|^{2}\cos ^{2}\psi  = \infty 
\end{equation}
\begin{equation}
\underset{t \rightarrow \infty }{lim}\frac{\beta _k(t)}{\kappa _k (t)} = 0 \\
\end{equation}
By applying Lem.\ref{lem:stability analysis}, we can conclude that 
\begin{equation*}
\underset{t \rightarrow \infty }{lim} h_k(t) = \underset{t \rightarrow \infty }{lim} M_{V_{k}}\|z_k - C_{V_{k}} \|^2 = 0
\end{equation*}
Since $ M_{V_{k}} $ is the mass of the Voronoi cell $k$, which is positive, then we have 
\begin{equation*}
\underset{t \rightarrow \infty }{lim} z_{k}(t) = C_{V_{k}}(t)
\end{equation*}

\textit{ 2) Absence of Zeno behavior:}

\begin{definition}
A solution is said to have a nonvanishing dwell time if there exists $\tau > 0$	such that
\begin{equation}
\inf_{l}(t_{l+1}^{k}- t_{l}^{k}) \geq \tau  
\end{equation}
\end{definition}

In other words, the existence of a lower bound $\tau $ for inter-execution time intervals is to be proved. Thus it can guarantee the system does not exhibit Zeno behavior. For any $l \geq 0$, and any $ k \in \left \{ 1,..., n \right \}$, consider the time interval $ t \in [t_{l}^{k}, t_{l+1}^{k}) $. From the definition of $ e_{k}(t) $ in Eq.\eqref{eq: actuation mismatch} and the fact that in the time interval $ u_{k}(t_{l}^{k})$ is a constant, we observe the time derivative of $\left | e_{k}(t)  \right | $ satisfies 
\begin{equation} \label{inequality dot e_k(t) }
\frac{\mathrm{d} }{\mathrm{d} t}\left | e_{k}(t)  \right | \leq \left | \dot{u}_{k} \right |
\end{equation}
From Eq.\eqref{eq:continuous control law}, we have 
\begin{equation*}
\left | \dot{u}_{k} \right | = \gamma w_{0}\left | \left \langle \dot{z_{k}}-\dot{C}_{V_{k}}, v_{k}e^{i\theta _{k}} \right \rangle + \left \langle z_{k}-C_{V_{k}}, iv_{k}u_{k}e^{i\theta _{k}} \right \rangle \right |  
\end{equation*}
then we have
\begin{equation} \label{inequality uk}
\begin{split}
\left | \dot{u}_{k} \right | &\leq  \gamma w_{0} \left \| \dot{z_{k}}-\dot{C}_{V_{k}} \right \|\left \| v_{k}e^{i\theta _{k}} \right \|  \notag \\
&~~~ + \gamma w_{0}\left \| z_{k}-C_{V_{k}} \right \|\left \|  v_{k}e^{i\theta _{k}} \right \|\left | u_{k} \right |   
\end{split}
\end{equation}
Define the following auxiliary variables:\\
\begin{equation*}
A(t) = \left \| \dot{z_{k}}-\dot{C}_{V_{k}} \right \|\left \| v_{k}e^{i\theta _{k}} \right \|
\end{equation*}

\begin{equation*}
B(t) = \left \| z_{k}-C_{V_{k}} \right \|\left \|  v_{k}e^{i\theta _{k}} \right \|\left | u_{k} \right |
\end{equation*}

The system dynamics are assumed to satisfy the following properties:
\begin{itemize}
\item P(1) There exists a scalar constant $d > 0 $ such that $ \left \| z_{k}-C_{V_{k}} \right \| < d $.\\
\item P(2) There exists positive constants $v_{k,nom}$ for each robot such that $v_{k} \leq v_{k,nom}$.\\
\end{itemize}

From Eq.\eqref{eq:continuous control law}, we obtain 
\begin{equation}
\left | u_{k} \right | \leq w_{0} + \gamma w_{0}\left \| z_{k}-C_{V_{k}} \right \|\left \|  v_{k}e^{i\theta _{k}}\right \|
\end{equation}
Then from the assumed properties P(1) and P(2), we have that $\left | u_{k} \right |$ is upper bounded by a positive value. It is straightforward to see that the term $ B(t) $ is bounded.\\

The time derivative of the centroid of the Voronoi cell is given by

\begin{equation} \label{eq: dot_C_Vk }
\dot{C}_{V_{k}}= \frac{\partial C_{V_{k}}}{\partial z_{k}}\dot{z_{k}} + \sum_{j \in N_{k}}^{}\frac{\partial C_{V_{k}}}{\partial z_{j}}\dot{z_{j}}
\end{equation}

\begin{equation}
\dot{z_{k}}-\dot{C}_{V_{k}} = \dot{z_{k}} - \frac{\partial C_{V_{k}}}{\partial z_{k}}\dot{z_{k}} - \sum_{j \in N_{k}}^{}\frac{\partial C_{V_{k}}}{\partial z_{j}}\dot{z_{j}}
\end{equation}
From Eq.\eqref{eq:event_based virtual_centre_dynamic} we have 
\begin{equation}
\left \| \dot{z}_k(t) \right \| \leq \left \| \gamma v_ke^{i\theta_k(t)} \right \|O+ \left \| \frac{v_k}{\omega_0}e^{i\theta_k(t)} \right \|\left | e_{k}(t) \right |
\end{equation}
According to the trigger condition we have $ \left | e_{k}(t) \right | \leq \left | \gamma \omega_0 g_{k}(z_{k}(t)) \right | + \mu_{k}(t) $. Together with the assumed properties we have that $\left \| \dot{z}_k(t) \right \|$ is upper bounded. Analysis similar to $\left \| \dot{z}_j(t) \right \|$. \\

We note that the work \cite{lee2015multirobot} provided the analytic solutions of the partial derivative $\frac{\partial C_{V_{k}}}{\partial z_{k}}$ and $ \frac{\partial C_{V_{k}}}{\partial z_{j}} $. We refer the readers to \cite{lee2015multirobot} for more details. 


In this paper, we consider the convex polygon in the planar case and that the area the robots need to cover is bounded. So the area of each Voronoi cell is bounded. It is evident that the integration area, in other words, the boundary of the Voronoi cell is bounded. Thus, $ \left \| \frac{\partial C_{V_{k}}}{\partial z_{k}} \right \| $ and $ \left \| \frac{\partial C_{V_{k}}}{\partial z_{j}} \right \| $ are bounded. Then we have
\begin{equation}
\left \| \dot{z_{k}}-\dot{C}_{V_{k}} \right \| \leq \left \| 1- \frac{\partial C_{V_{k}}}{\partial z_{k}} \right \|\left \| \dot{z_{k}} \right \|+ \sum_{j \in N_{k}}^{} \left \| \frac{\partial C_{V_{k}}}{\partial z_{j}} \right \|\left \| \dot{z_{j}} \right \|
\end{equation}

The term $\left \| \dot{z_{k}}-\dot{C}_{V_{k}} \right \| $ is bounded. Then we obtain that $B(t)$ is bounded. From the above conclusions, it is straightforward to conclude that $\left | \dot{u}_{k} \right |$ is bounded.
Define a positive constant $B_{e}$, which represents the upper bound of $ \left | \dot{u}_{k} \right | $. Then, we obtain
\begin{equation*}
\frac{\mathrm{d} }{\mathrm{d} t}\left | e_{k}(t)  \right | \leq B_{e}
\end{equation*}

It follows that 
\begin{equation} \label{inequality: e_k integral}
\left | e_{k}(t)  \right | \leq \int_{t_{l}^{k}}^{t}B_{e} dt = (t-t_{l}^{k}) B_{e}
\end{equation}
for $t \in [ t_{l}^{k}, t_{l+1}^{k} )$ and for any $l$. Recall the trigger function Eq.\eqref{eq: trigger function}, and the control input mismatch $e_{k}(t)$ is reset to zero at $t_{l}^{k} $. It follows that the next event time $t_{l+1}^{k}$ is determined by the changing rates of $e_{k}(t)$ and the threshold $ \sigma \gamma \omega_0 \left | g_{k}(z_{k}(t)) \right | + \mu_{k}(t)$. It is evident that 

\begin{equation}
\left | e_{k}(t)  \right | = \sigma \gamma \omega_0 \left | g_{k}(z_{k}(t)) \right | + \mu_{k}(t)
\end{equation}
holds at the next trigger time $t_{l+1}^{k}$. In the stability analysis part we conclude that  $ \lim_{t\rightarrow \infty }z_{k}(t) = C_{V_{k}}(t) $. However, it is important to point out that in the evolution of the system, $z_{k}(t) - C_{V_{k}}(t) = 0 $ may also hold at $t_{l+1}^{k}$ while the optimal coverage is not achieved as $\dot{z_{k}}$ can be nonzero at $t_{l+1}^{k}$. We consider the triggering at $ t_{l+1}^{k} $ in the following two cases:

\begin{itemize}
	\item Case 1: If $\left \| z_{k}(t_{l+1}^{k}) - C_{V_{k}}(t_{l+1}^{k}) \right \|  \neq 0 $, the equality $\left | e_{k}(t_{l+1}^{k})  \right | = \sigma \gamma \omega_0 \left | g_{k}(z_{k}(t_{l+1}^{k})) \right | + \mu_{k}(t_{l+1}^{k})$ is satisfied.
	\item Case 2: If $\left \| z_{k}(t_{l+1}^{k}) - C_{V_{k}}(t_{l+1}^{k}) \right \|  = 0 $, the equality $\left | e_{k}(t_{l+1}^{k})  \right | = \mu_{k}(t_{l+1}^{k})$ is satisfied.
\end{itemize}
It is evident to see that $\left \| z_{k}(t_{l+1}^{k}) - C_{V_{k}}(t_{l+1}^{k}) \right \| > 0$ for any $\left \| z_{k}(t_{l+1}^{k}) - C_{V_{k}}(t_{l+1}^{k}) \right \|  \neq 0 $. Compare the above two cases, we can conclude that it takes longer for the quantity $\left | e_{k}(t_{l+1}^{k})  \right |$ to increase to be equal to the quantity $\left | \gamma \omega_0 g_{k}(z_{k}(t_{l+1}^{k})) \right | + \mu_{k}(t_{l+1}^{k})$ than to increase to be equal to the quantity $\mu_{k}(t_{l+1}^{k})$. In other words, $\tau_{Case 1} > \tau_{Case 2} $. \\

According to Eq.\eqref{inequality: e_k integral},
\begin{equation}
B_{e}\tau_{Case 2} \geq \left | e_{k}(t)  \right | = \mu_{k}(t) = \gamma \omega_0 e^{-\alpha _{k}( t_{l}^{k} + \tau_{Case 2} )}
\end{equation}
where $ \gamma \omega_0 e^{-\alpha _{k}( t_{l}^{k} + \tau_{Case 2} )} > 0$, $B_{e}$ is a positive constant. We can conclude that the inter-event time interval $\tau_{Case 2}$ is strictly positive. Since there is a positive lower bound on the inter-event time intervals, there are no accumulation points in the event sequences, so Zeno behavior is excluded for all robots.

\end{proof}

\bibliographystyle{ieeetr}
\bibliography{coverage}

\end{document}